\documentclass[runningheads]{llncs}

\usepackage[utf8]{inputenc}
\usepackage[T1]{fontenc}
\usepackage[english]{babel}
\usepackage{graphicx}
\usepackage{amssymb}
\usepackage{here}
\usepackage{tabularx}
\usepackage{amsmath}
\usepackage{cite}
\usepackage{academicons}
\usepackage{xcolor}

\usepackage{tikz,xcolor,hyperref}

\definecolor{lime}{HTML}{A6CE39}
\DeclareRobustCommand{\orcidicon}{
	\begin{tikzpicture}
	\draw[lime, fill=lime] (0,0) 
	circle [radius=0.16] 
	node[white] {{\fontfamily{qag}\selectfont \tiny ID}};
	\draw[white, fill=white] (-0.0625,0.095) 
	circle [radius=0.007];
	\end{tikzpicture}
	\hspace{-2mm}
}

\foreach \x in {A, ..., Z}{%
	\expandafter\xdef\csname orcid\x\endcsname{\noexpand\href{https://orcid.org/\csname orcidauthor\x\endcsname}{\noexpand\orcidicon}}
}

\newcommand{\NNR}{{\sc NNReach}}
\newcommand{\VIP}{{\sc VIP}}
\newcommand{\NECE}{{\sc NECE}}
\newcommand{\ANECE}{{\sc ANECE}}

\newcommand{\MIN}{{\sc MIN}}
\newcommand{\NE}{{\sc NE}}
\newcommand{\R}{{\mathbb{R}}}

\DeclareMathOperator{\CR}{CR}
\DeclareMathOperator{\aCR}{aCR}
\DeclareMathOperator{\GSR}{glob-SR}
\DeclareMathOperator{\GLR}{glob-LR}
\DeclareMathOperator{\GACR}{glob-aCR}
\DeclareMathOperator{\SR}{SR}
\DeclareMathOperator{\LR}{LR}
\DeclareMathOperator{\mVIP}{VIP}
\DeclareMathOperator{\mNNR}{NNReach}
\DeclareMathOperator{\mANECE}{ANECE}
\DeclareMathOperator{\mMIN}{MIN}
\DeclareMathOperator{\mNE}{NE}
\DeclareMathOperator{\mNECE}{NECE}

\begin{document}


\title{Robustness Verification in Neural Networks}
\author{Adrian Wurm\orcidA{}}

\authorrunning{A. Wurm}

\institute{BTU Cottbus-Senftenberg, Lehrstuhl Theoretische Informatik,\\ Platz der Deutschen Einheit 1, 03046 Cottbus, Germany,
\url{https://www.b-tu.de/}
\email{wurm@b-tu.de}}
\maketitle              

\textbf{Abstract:} In this paper we investigate formal verification problems for Neural Network computations. Of central importance will be various robustness and minimization problems such as: Given symbolic specifications of allowed inputs and outputs in form of  Linear Programming instances, one question is whether there do exist valid inputs such that the network computes a valid output? And does this property hold for all valid inputs? Do two given networks compute the same function? Is there a smaller network computing the same function? 

The complexity of these questions have been investigated recently from a practical point of view and approximated by heuristic algorithms. We complement these achievements by giving a theoretical framework that enables us to interchange security and efficiency questions in neural networks and analyze their computational complexities. We show that the problems are conquerable in a semi-linear setting, meaning that for piecewise linear activation functions and when the sum- or maximum metric is used, most of them are in P or in NP at most.

\section{Introduction}
Neural networks are widely used in all kinds of data processing, especially on seemingly unfeasible tasks such as image\cite{Krizhevsky} and language recognition\cite{Hinton}, as well as applications in medicine \cite{Litjens}, and prediction of stock markets \cite{Dixon}, just to mention a few. 
Khan et al. \cite{Khan} provide a survey of such applications, a mathematically oriented textbook concerning structural issues related to Deep Neural Networks is provided by \cite{Calin}.

Neural networks are nowadays also made use of in safety-critical systems like autonomous driving\cite{Grigorescu} or power grid management. In such a setting, when security issues become important, aspects of certification come into play \cite{Huang, dreossi,mahloujifar}. If we for example want provable guarantees for certain scenarios to be unreachable, we first need to formulate them as constraints and precisely state for which property of a network we want verification.

In the present paper we are interested in studying certain verification problems
for NNs in form of particular robustness and minimization problems such as: How will a network react to a small perturbation of the input\cite{occrob}? And how likely is a network to change the classification of an input that is altered a little? These probabilities are crucial when for example a self-driving car is supposed to recognize a speed limit, and they have already been tackled in practical settings by simulations and heuristic algorithms. These approaches however will never guarantee safety, because the conditions under which seemingly correct working nets suddenly tend to decide irrationally often seem arbitrary and unpredictable \cite{athalye,carlini,uesato}. 

Key results will be that a lot of these questions behave very similar under reasonable assumptions. Note that the network in principle is allowed to compute with real numbers, so
the valid inputs we are looking for belong to some space $\mathbb R^n$, but the network
itself is specified by its discrete structure and rational weights defining the linear
combinations computed by its single neurons.

A huge variety of networks arises when changing the underlying activation functions. There are of course many activations frequently
used in NN frameworks, and in addition we could extend verification
questions to nets using all kinds of activation. One issue to be discussed is the
computational model in which one argues. If, for example, the typical sigmoid
activation $f(x) = 1/(1 + e^{-x})$ is used, it has to be specified in which sense it is
computed by the net: For example exactly or approximately, and at which costs
these operations are being performed. In this paper we will cover the most common activation functions, especially $ReLU$.

The paper is organized as follows:
In Section \ref{Section:preliminaries} we collect basic notions, recall the definition of feedforward neural nets as used in this paper as well as for the metrics that we use. Section \ref{Section:robustness} studies various robustness properties and their comparison to each other when using different network structures and metrics. In Section \ref{Section:mini}, we give criteria on whether a network is as small as possible, or if it can be replaced by a smaller network that is easier to evaluate.

The paper ends with some open questions. 

\section{Preliminaries and Network Decision Problems}\label{Section:preliminaries}

We start by defining the problems we are interested in; here,
we follow the definitions and notions of \cite{adrian} and \cite{lange} for everything related to neural networks. The networks considered are exclusively feedforward. In their
most general form, they can process real numbers and contain rational weights. This will
later on be restricted when necessary.

\begin{definition} A (feedforward) neural network N is a layered graph that represents a
function of type $\mathbb R^n \rightarrow \mathbb R^m$, for some $n,m \in \mathbb N$.
The first layer with label $\ell = 0$ is called the \emph{input layer} and consists of $n$ nodes called \emph{input nodes}. 
The input value $x_i$ of the $i$-th node is also taken as its output $y_{0i}  := x_i$.
A layer $1 \leq \ell \leq L - 2$ is called \emph{hidden} and consists of $k(\ell)$ nodes called \emph{computation nodes}. The $i$-th node of layer
$\ell$ computes the output $y_{{\ell}i} = \sigma_{{\ell}i} (\sum\limits_j c_{ji}^{({\ell}-1)}y_{({\ell}-1)j} + b_{{\ell}i}))$. 
Here, the $\sigma_{{\ell}i}$ are (typically nonlinear)  \emph{activation} functions (to be specified later on) and the
sum runs over all output neurons of the previous layer. The $c^{({\ell}-1)}_{ji}$ are real constants which are called \emph{weights}, and $b_{{\ell}i}$ is a real constant  called \emph{bias}.
The outputs of all nodes of layer $\ell$ combined
gives the output $(y_{l0}, . . . , y_{l(k-1)})$ of the hidden layer.
The final layer $L - 1$ is called \emph{output layer} and consists of $m$ nodes called \emph{output nodes}.
The $i$-th node computes an output $y_{(L-1)i}$ in the same way as a node in a hidden
layer. The output $(y_{(L-1)0}, . . . , y_{(L-1)(m-1)})$ of the output layer is considered
the output $N(x)$ of the network $N$.
\end{definition}

Note that above we allow several different activation functions in a single network. This basically is because for some results technically the identity is necessary as a second activation function beside the 'main' activation function used. All of our results hold for this scenario already.

Since we want to study its complexity in the Turing model, we restrict all weights and
biases in a NN to be rational numbers. The problem \NNR\ involves two Linear Programming LP instances in a decision version, therefore recall that such an instance consists of a system of (componentwise) linear
inequalities $A\cdot x \leq b$ for rational matrix $A$ and vector $b$ of suitable dimensions.
The decision problem asks for the existence of a real solution vector $x.$ As usual, we refer to this problem as Linear Program Feasibility LPF. By abuse of notation, we denote by $A$ also the set described by the pair $(A,b)$.

As usual for neural networks, we consider different choices for the activation functions used, but concentrate the most frequently used one, namely $ReLU(x)=max\{0,x\}$. We name nodes after their internal activation function, so we call nodes with activation function $\sigma(x)=x$ identity nodes and nodes with activation function $\sigma(x)=ReLU(x)$ $ReLU$-nodes for example.

We next recall from \cite{adrian} the definition of the decision problems \NNR, \VIP\ and \NE.
These problems have been investigated in \cite{Huang} and \cite{Albarghouthi}, specifically \NNR\ in \cite{adrian}, \cite{lange} and  \cite{Katz}.

\begin{definition}
a) Let $F$ be a set of activation functions from $\mathbb{R}$ to $\mathbb{R}$.  
An instance of the \emph{reachability problem for neural networks} \NNR$(F)$  
consists of a (feedforward) neural network $N$ with all its activation functions belonging to $F$, rational data as
weights and biases, and 
two instances $A$ and $B$ of LP in decision version with rational data, 
one with the input variables of $N$ as variables, and the other 
with the output variables of $N$ as variables. 
These instances are also called \emph{input} and \emph{output specifications}, respectively. 
The problem is to decide if there exists an $x \in \mathbb R^n$ that satisfies the input specifications such that the output $N(x)$ 
satisfies the output specifications. 

b) The problem \emph{verification of interval property} $\mVIP(F)$ consists of the same instances. The question is whether for all $x \in \mathbb{R}^n$ satisfying the input specifications, $N(x)$ will satisfy the output specifications.

c) The problem \emph{network equivalence} $\mNE(F)$ is the question whether two $F$-networks $N_1$ and $N_2$ compute the same function or not.

d) The size of a network is $T \cdot L$; here, $T$ denotes the number of neurons in the net $N$ and $L$ is the
maximal bit-size of any of the weights and biases. The size of an instance of \NNR, \VIP\ or \NE\ is sum of the network size and the usual bit-sizes of the LP-instances.
\end{definition}

We omit $F$ in the notation if it is obvious from the context and write \VIP$(N,A,B)$ if \VIP$(F)$ holds for the instance $(N,A,B)$. The notations for the other problems are similar.

We additionally define certain robustness properties that were examined for example in \cite{Ruan} and use them to assess the computed function in a more metric way similar to \cite{Casadio}.

\begin{definition} Let $N$  be a NN, $N_i$ the projection of $N$ on the $i$-th output dimension, $1\leq j\leq m$ an output dimension, $\bar x\in\mathbb R^n$, $\varepsilon,\delta\in\mathbb Q_{\geq0}\cup\{\infty\}$ and $d$ a metric on $\R^n$ as well as $\R^m$.

a) The classification of a network input $x$ is the output dimension $1\leq i\leq m$ with the biggest value $N_i(x)$, which is interpreted as the attribute that the input is most likely to have. We say that $N$ has $\varepsilon$-\emph{classification robustness} in $\bar x$ with respect to $d$ iff the classification of every input $x$ that is $\varepsilon$-close to $\bar x$ is the same as for $\bar x$, meaning 
\[\CR_d(N,\varepsilon,\bar x,j):\Leftrightarrow\forall  x: d(x,\bar x)\leq\varepsilon\Rightarrow \text{arg}\,\max\limits_{i}N_i(x)=j\]

and \emph{arbitrary} $\varepsilon$-\emph{classification robustness}
iff
\[\aCR_d(N,\varepsilon,\bar x):\Leftrightarrow\exists j\in\{1,...,m\}:\CR_d(N,\varepsilon,\bar x,j)\]

b) We say that $N$ has $\varepsilon$-$\delta$-\emph{standard robustness} in $\bar x$ with respect to $d_\lambda$ iff $N$ interpreted as a function is $\varepsilon$-$\delta$-continuous in $\bar x$, meaning
\[\SR_d(N,\varepsilon,\delta,\bar x):\Leftrightarrow\forall  x: d(x,\bar x)\leq\varepsilon\Rightarrow d(N(\bar x),N(x))\leq\delta\]

c) We say that $N$ has $\varepsilon$-$L$-\emph{Lipschitz robustness} in $\bar x$ with respect to $d$ iff
\[\LR_d(N,\varepsilon,L,\bar x):\Leftrightarrow\forall  x: d(x,\bar x)\leq\varepsilon\Rightarrow d(N(\bar x),N(x))\leq L\cdot d(\bar x,x)\] 

d) For a specific set $F$ of activation functions, by abuse of notation we denote by $\CR_d(F)$ the decision problem whether $\CR_d(N,\varepsilon,\bar x,j)$ holds where we allow as input only networks $N$ that use functions from $F$ as activations, the notations for the other problems are defined similarly.

We denote by $d_p$ the metric induced by the $p$-Norm, meaning $d_p(x,y)=(\sum\limits_{i=1}^n (x_i-y_i)^p)^{\frac 1p}$.
\end{definition}

Note that $d_1$ and $d_\infty$ share the property that all $\varepsilon$-balls $B_{\varepsilon,d}(x)$ are semi-linear sets and that they can also be described as instances of LPF.

The property $\aCR$ seems redundant and unnatural at first, it is however not equivalent to $\CR$ at an extent that one could hastily assume. Note that it is in general already very hard to calculate the solution of just one network computation $N(x)$ for given $N$ and $x$ for certain activation functions, the result does not even have to be rational or algebraic any more. It might even be impossible to determine the biggest output dimension, i.e., to tell whether $N(x)_i\geq N(x)_j$, for in the Turing model, one is not capable of precisely computing certain sigmoidal functions, which is why seemingly obvious reductions from $\CR$ to $\aCR$ fail. Assume for example that an activation similar to the exponential function is included, in this case the problem might become hard for Tarskis exponential problem, which is not even known to be decidable. Checking the properties $\aCR$ and $\CR$ is therefore not necessarily of the same computational complexity. We also need $\aCR$ in reduction proofs later on.   

\section{Complexity Results for Robustness}\label{Section:robustness}

In network decision problems, $d_1$ and $d_\infty$ behave fundamentally different than the other $p$-metrics, for they are the only ones that lead to linear constraints. General polynomial systems tend have a higher computational complexity than linear ones, compare for example $\exists\R$ (NP-hard) with LP (in P) or Hilberts tenth problem (undecidable) with integer linear systems (NP-complete). The following proposition partially covers the linear part of our network decision problems.

\begin{proposition}\label{VIPcoNP}
Let $F$ be a set of piecewise linear activation functions. 
\begin{description}
\item[a)] The problems $\mVIP(F)$ and $\mNE(F)$ are in co-NP

\item[b)] Let $d\in\{d_1,d_\infty\}$. Then  $\SR_d(F)$, $\CR_d(F)$ and $\LR_d(F)$ are in co-NP.
\end{description}
This is specifically the case for $F=\{ReLU\}$.
\end{proposition}

\begin{proof} Finding a witness for violation of the conditions works essentially the same for all the above problems, we will perform it for $\mVIP$.

The idea of the proof is that an instance is not satisfiable iff there exists an $x\in\mathbb R ^n$ fulfilling the input specifications so that $N(x)$ does not fulfill the output specifications, meaning at least one of the equations does not hold. We then guess this equation along with some data that narrows down how the computation on such an $x$ could behave.

For a no-instance, let $x$ be an input violating the specifications. The certificate for the instance not to be solvable is not $x$ itself, because we would need to guarantee that it has a short representation. Instead, the certificate provides the information for every computation node, in which of the linear intervals its input is (for ReLU-nodes for example whether it returns zero or a positive value) when the network computes $N(x)$ together with the index of the violated output specification. This leaves us with a system like Linear Programming of linear inequalities that arise in the following way: 

All of the input specifications $\sum\limits_{i=1}^na_ix_i\leq b$ and the violated output specification $\sum\limits_{i=1}^ma_iN(x)_i> b$ are already linear. For each $\sigma$-node computing
\[y_{{\ell}i} = \sigma (\sum\limits_j c_{ji}^{({\ell}-1)}y_{({\ell}-1)j} + b_{{\ell}i}))\] that is by the certificate on an interval $[\alpha,\beta]$ where $\sigma(x)$ coincides with $a\cdot x +b$, we add the linear conditions 
\[\alpha\leq \sum\limits_j c_{ji}^{({\ell}-1)}y_{({\ell}-1)j} + b_{{\ell}i}\leq\beta\]

and
\[y_{{\ell}i} =a\cdot \sum\limits_j c_{ji}^{({\ell}-1)}y_{({\ell}-1)j} + b_{{\ell}i}+b\]

The only difference to linear programming is that some of the equations are strict and some are not. This decision problem is also well known to still be in P by\cite{JonssonB}, Lemma 14.

The proof works in the same way for the other problems, as the falsehood of their instances can also formulated as a linear programming instance. For network equivalence for example, one would guess the output dimension that is unequal at some point, encode the computation of both networks, and demand that the outputs in the guessed dimensions are not the same, which is a linear inequality.$\hfill\blacksquare$
\end{proof}

The proof fails for non-linear metrics such as $d_2$, for the resulting equation system is not linear any more. 



\begin{lemma}
Let $F$ be a set of activations containing ReLU. For any $\lambda\in[1,\infty]$, there are polynomial time reductions from $d=d_\infty$ or $d=d_1$ to $d=d_\lambda$ for $\CR_d(F), \aCR_d(F)$ and $\SR_d(F)$.

\end{lemma}

\begin{proof} We will show how to reduce from $d_1$ for a network $N$ by constructing a network $N'$ that shows the same behavior in $\mathbb B_{\varepsilon,1}(\bar x)$ as the given network and a sufficiently similar behavior on $\mathbb B_{\varepsilon,\lambda}(\bar x)$. The proof for $d_\infty$ is analogous, as it only relies on the $\varepsilon$-ball to be a convex semi-linear set.

The idea for all three problems is to first retract a superset of the $\varepsilon$-ball to it so that $N'=N\circ T$, where $T(B_{\varepsilon,\lambda}(\bar x))=B_{\varepsilon,1}(\bar x)$.

Let $(N,\varepsilon,\bar x,j)$ be an instance of $\CR_1(F)$. The equivalent instance $(N',\varepsilon,\bar x,j)$ of $\CR_\lambda(F)$ is constructed as follows:

We first insert $n$ new layers before the initial input layer, the first of them becomes the new input layer. These layers perform a transformation $T$ that has the following properties:

$\forall x\in\mathbb B_{\varepsilon,1}(\bar x): T(x)=x$. This ensures that if the initial instance failed the property $\CR_1(F)$, the new one fails $\CR_\infty(F)$.

$\forall x\in\mathbb B_{\varepsilon,\lambda}(\bar x): T(x)\in \mathbb B_{\varepsilon,1}(\bar x)$. This ensures that if the initial instance fulfilled the property $\CR_1(F)$, the new one fulfills $\CR_\infty(F)$. 

Such a function is for example \begin{align*}
T(x)_i=&ReLU(x_i)-ReLU(x_i+\sum\limits_{j=1}^{i-1}\vert T(x)_j\vert -\varepsilon)\\
-&(ReLU(-x_i)+ReLU(-x_i+\sum\limits_{j=1}^{i-1}\vert T(x)_j\vert -\varepsilon))
\end{align*}

because $\mathbb B_{\varepsilon,1}(\bar x)\subseteq\mathbb B_{\varepsilon,\lambda}(\bar x)$ and $T(x)\in\mathbb B_{\varepsilon,1}(\bar x) \forall x\in\mathbb R^n$.

For $\CR$ the reduction is complete, for $\aCR$ it is essentially the same. For $\SR_\lambda$ we perform as one further step a transformation $P$ of $N(x)$ into a one-dimensional output, namely $P(N(x))=\sum\limits_{i=1}^m\vert N(x)_i\vert$. It is easily seen that $N(x)\in\mathbb B_1(\varepsilon)\Leftrightarrow P(N(x))\in\mathbb B_1(\varepsilon)=[0,\varepsilon]$.
$\hfill\blacksquare$
\end{proof}

The previous Theorem as well as the following Proposition rely on the semi-linearity of the final sets, a crucial property when using the Ellipsoid Method or the ReLU function. When . This is the reason why we are restricted to metrics where balls are semi-linear sets. When moving on to $d_\lambda$ for arbitrary $\lambda$, we cross the presumably large gap in computational complexity between $NP$ and $\exists\mathbb R$.

\begin{proposition}\label{P}
Let $\lambda\in\{1,\infty\}$. Then $\mVIP(\{id\})$, $\mNNR(id)$, $\mNE(id)$,\\
$\SR_\lambda(\{id\})$, $\CR_\lambda(\{id\})$, $\aCR_\lambda(\{id\})$ and $\LR_\lambda(\{id\})$ are in P.
\end{proposition}

\begin{proof}
$\NNR(id)$ is a special case of Linear Programming Feasibility. For the other problems, we argue that their complements are in P=co-P. Note that LPF is still in P if we allow strict linear inequalities. The complements of $\VIP(\{id\})$, $\NE(id)$, $\SR_\lambda(\{id\})$, $\CR_\lambda(\{id\})$, $\aCR_\lambda(\{id\})$ and $\LR_\lambda(\{id\})$ are global disjunctions of LP-instances of linear size, so all of these Problems are in P.$\hfill\blacksquare$
\end{proof}

The following Lemma proved in \cite{adrian} allows us to strengthen a couple of results on complexities in the semi-linear range of network verification.

\begin{lemma}
An identity node can be expressed by two ReLU nodes without changing the computed function.
\end{lemma}

The Lemma looks somewhat trivial, however for all our decision problems it shows that adding $id$ to the set $F$ of allowed activations already containing ReLU does not change the computational complexity. Furthermore, when reducing one such problem to another one, we can assume $id$ to be among the activations and freely use it in the reductions, as for example in the next Lemma.

\begin{lemma}Let $F$ be a set of activation functions.

i) $\aCR_d(F)$ is polynomial time reducible to $\CR_d(F)$.

ii) If $d$ is induced by a $p$-norm and if $ReLU\in F$, then $\CR_d(F)$  is polynomial time reducible to $\aCR_d(F)$.

\end{lemma}

\begin{proof}
i) To see if $aCR_d(N,\varepsilon,\bar x)$ holds, check $\CR_d(N,\varepsilon,\bar x,j)$ for all possible $j$ and see if one of them holds. 

ii) For the reverse direction, let $(N,\varepsilon,\bar x,j)$ be an instance of $\CR_d(F)$. We will construct an instance $(N_h,\varepsilon,\bar x)$ of $aCR$ so that 
\[\CR_d(N,\varepsilon,\bar x,j)\Leftrightarrow \aCR_d(N_h,\varepsilon,\bar x)\]
The idea is to apply a function $f$ on the output that collects any possible overlap of all dimensions other than $j$. This function is then modified to a function $h$ with 3 output dimensions described by a network $N_h$, where one dimension is the constant zero function, and it is guaranteed to be negative in both remaining dimensions at some point. Also, iff $f$ was at some point positive, which happens exactly if there was an overlap, so should at least one component of $h$ be. Since $h$ is at some point negative in the second and third dimension, the only candidate for the pointwise biggest output dimension is the first, and this is the case iff there was an overlap in the output of the initial network. This means that $\CR_d(N,\varepsilon,\bar x,j)\Leftrightarrow\CR_d(N_h,\varepsilon,\bar x,1)\Leftrightarrow\aCR_d(N_h,\varepsilon,\bar x)$. It remains to show that such functions $f,h$ exist and that they can be described by a network $N_h$. 

Note that the following functions can be computed by neural networks that are modifications of $N$: 
\[f(x)=\sum\limits_{i=1}^m ReLU(N(x)_i-N(x)_j)\]
\[g(x)=f(x)-ReLU(f(x)-1)\]

Note that $f(x)\geq0$ and $f(x)=0$ iff $j$ is the index for a maximal component. Furthermore if $f(x)\geq1$, then $g(x)=1$ and if $f(x)\in[0,1]$ then $g(x)=f(x)$. Thus, $g\equiv0$ iff $\CR_d(N,\varepsilon,\bar x,j)$. Now define $h$ as

\[h(x)=(0,g(x)-\frac2\delta ReLU(x_1-\bar x_1),g(x)-\frac2\delta ReLU(\bar x_1-x_1)).\]
Where $\delta\in\mathbb Q $ is chosen so that $\delta<\varepsilon$, meaning $\bar x+\delta e_1, \bar x-\delta e_1\in\mathbb B_{\varepsilon,d}(\bar x)$ where $e_1=(1,0,...,0)$.

Let $N_h$ be the network computing $h$. When asked if $aCR_\lambda(N_h,\varepsilon,\bar x)$, the second and third dimensions of $h$ cannot be the pointwise biggest, the only possible candidate is 0 (j=1) because $h_2(\bar x+\delta e_1),h_3(\bar x-\delta e_1)\leq -1$. Also, $g$ is constantly 0 iff $\CR_\lambda(N,\varepsilon,\bar x,j)$ holds, on the other hand $g\equiv 0$ holds iff $aCR_\lambda(N_h,\varepsilon,\bar x)$, so $\CR_\lambda(N,\varepsilon,\bar x,j)\Leftrightarrow aCR_\lambda(N_h,\varepsilon,\bar x)$.$\hfill\blacksquare$

\end{proof}

\begin{theorem}\label{SRCRVIP} Let $F$ be a set of activation functions. We have that

i) $\SR_\infty$($F$) is linear time reducible to $\VIP(F)$,

ii) $\CR_\infty$($F$) is linear time reducible to $\VIP(F)$,

iii) if $id\in F$, then $\SR_\infty$($F$) is linear time reducible to $\CR_\infty(F$),

iv) if $ReLU\in F$, then $\CR_d(F$) is linear time reducible to $\SR_d(F$) for any metric $d$.

\end{theorem}

\begin{proof}
i) Let $(N,\varepsilon,\delta,\bar x)$ be an instance of SR$_\infty$($F$). To construct an equivalent instance of $\VIP(F)$, we introduce two copies of the network $N$ and view them as one big network $N'$ with an input dimension twice the input dimension of $N$. We demand by input specification that the input of the first copy of $N$ is $\bar x$ and the input of the second copy is in the $\varepsilon$-ball of $\bar x$ with respect to $\|\cdot\|_\infty$, which is equivalent to $\forall i\in\{1,...,n\}: -\varepsilon\leq \bar x_i-x_i\leq\varepsilon$, a linear equation system that we denote $A$. Similarly, we demand by output specification that $N(x)$ is in the  $\delta$-ball of $N(\bar x)$, which is equivalent to $\forall i\in\{1,...,m\}: -\delta\leq N(\bar x)_i-N(x)_i\leq\delta$, we denote this system $B$. This gives us the $\VIP(F)$-instance $(N',A,B)$ for which we have $\VIP(N',A,B)\Leftrightarrow \SR(N,\varepsilon,\delta,\bar x)$.
 
ii) Let $(N,\varepsilon,\bar x,j)$ be an instance of $\CR_\infty(F)$. We construct an equivalent instance of $\VIP(F)$ as follows. The network remains $N$, the input specifications are $\forall i\in\{1,...,n\}: \bar x_i-\varepsilon\leq x_i\leq\bar x_i+\varepsilon$ and the output specification is $\bigwedge\limits_{i=1}^m N(x)_i \leq N(x)_j$.

iii) Let $(N,\varepsilon,\delta,\bar x)$ be an instance of $\SR_\infty$($F$) with $m$ the output dimension of $N$. To construct an equivalent instance $(N',\varepsilon,\bar x,2m+1)$ of $\CR_\infty(F)$, we perform the computation of two copies of the network $N$ in parrallel, which constitutes a new network $N'$ as follows: 

\begin{figure}

  \centering
    \includegraphics[width=1\textwidth]{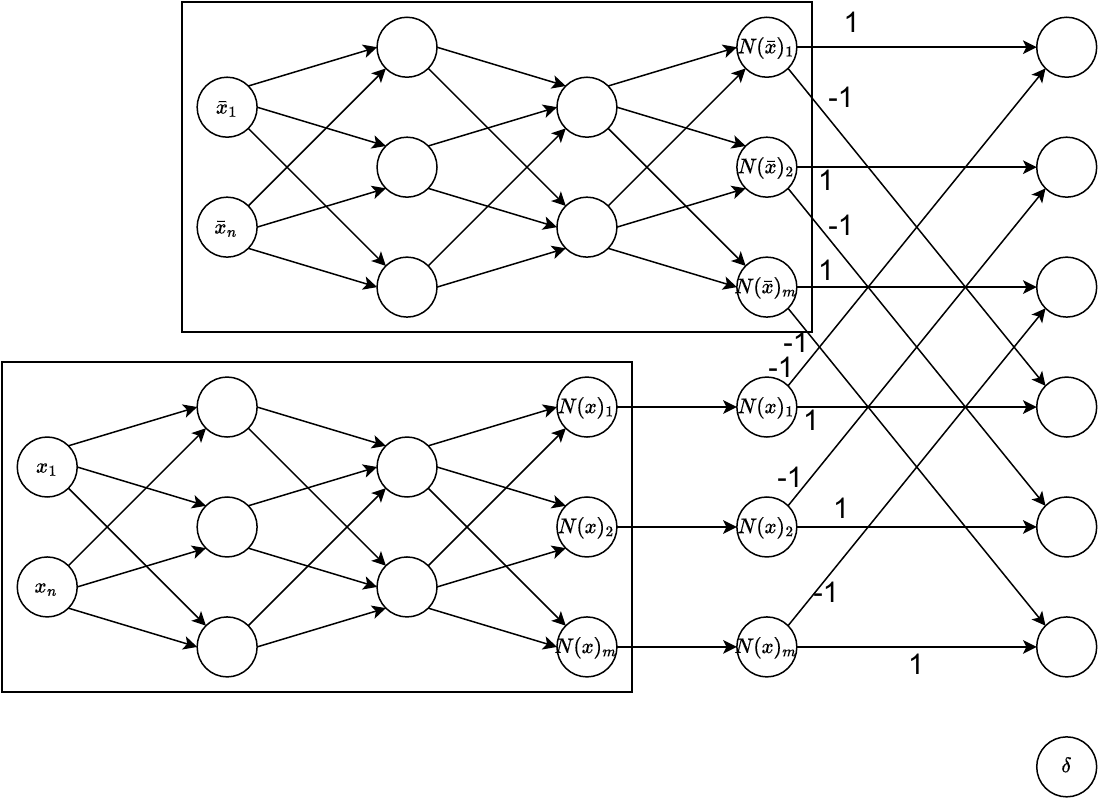}
  \caption{Idea iii)}
\end{figure}

The first copy is shifted by one layer, so the layer that was the input layer is now in the first hidden layer with no incoming connections and the bias of node $i$ is $\bar x_i$, so that first copy just computes the value $N(\bar x)$. The input layer of $N'$ is the input layer of the second copy, so the second copy of $N$ computes $N(x)$. We add a layer of id-nodes at the end of it that does not change the output but let it have the same number of layers as the first copy. We add one layer connecting both copies after the old output layer, this one becomes the new output layer, and it computes
\[(N(x)_1-N(\bar x)_1,...,N(x)_m-N(\bar x)_m, N(\bar x)_1-N(x)_1,...,N(\bar x)_m-N(x)_m,\delta).\]
The node corresponding to the last value $N'(x)_{2m+1}=\delta$ has no incoming connections and bias $\delta$. We have that $\forall i\in\{1,...,m\} : \vert N(\bar x)_i-N(x)_i\vert\leq\delta$ iff $\forall i\in\{1,...,m\} : N(\bar x)_i-N(x)_i\leq\delta\land N(x)_i-N(\bar x)_i\leq\delta$, so the two instances are in fact equivalent.

iv) Let $(N,\varepsilon,\bar x,j)$ be an instance of $\CR_d(F)$. We construct an equivalent instance $(N',\varepsilon,0,\bar x)$ of $\SR_d(F)$ as follows: We add two additional layers after the old output layer. The first one computes $\alpha_i=ReLU(N(x)_i-N(x)_j)$ for all $i$, these values are all 0 everywhere on $B_{d,\varepsilon}(\bar x)$ iff $\CR_d(N,\varepsilon,\bar x,j)$ holds. The next added layer computes $\beta=ReLU(\sum\limits_{i=1}^m\alpha_i)$, this value is 0 everywhere on $B_{d,\varepsilon}(\bar x)$ iff $\CR_d(N,\varepsilon,\bar x,j)$ holds. Now we have that $\CR_d(N,\varepsilon,\bar x,j)\Leftrightarrow\SR_d(N',\varepsilon,0,\bar x)\land N'(\bar x)=0$, so it remains to encode that last condition as a $\SR_d(F)$-instance. The function 
\[f(x):=\sum\limits_{i=1}^nReLU(x_i-\bar x_i)+ReLU(\bar x_i-x_i)\]
is 0 in $\bar x$ and strictly positive everywhere else and there exists an $F$-network $N_f$  computing $f$. The networks $N'$ and $N_f$ can be combined to an $F$-network $N''$ computing $g(x)=ReLU(N'-N_f)(x)$. If $\SR_d(N',\varepsilon,0,\bar x)$ holds, then 
\[N'(\bar x)=0\Leftrightarrow g(x)=0 \forall x\in B_\varepsilon(\bar x)\Leftrightarrow \SR_d(N'',\varepsilon,0,\bar x).\]
The conditions $\SR_d(N',\varepsilon,0,\bar x)$ and $\SR_d(N'',\varepsilon,0,\bar x)$ can be formalized as a single $\SR_d$ instance by merging both nets $N'$ and $N''$ into a single larger one, so the reduction is in fact many-one.$\hfill\blacksquare$
\end{proof}

\begin{proposition}\label{NECR} Let $F$ be a set of activation functions. 
If $id\in F$, then $\NE(F)$ is linear time reducible to $\CR_d(F)$ for any metric $d$.
\end{proposition}

\begin{proof}
Let $(N,N')$ be an instance of $\NE(F)$. We merge both networks $N$ and $N'$ into one network $N''$ by identifying/contracting their corresponding inputs and subtracting their corresponding outputs both ways to obtain the output layer 
\[(N(x)_1-N'(x)_1,...,N(x)_m-N'(x)_m, N'(x)_1-N(x)_1,...,N'(x)_m-N(x)_m,0).\] 
The last output node is always 0, this can be forced by introducing an id-node that has bias 0 and no incoming connections. Now 
\[\NE(N,N')\Leftrightarrow\CR_d(N'',\infty,0,2m+1),\]
which concludes the proof. $\hfill\blacksquare$

\end{proof}

Note that for Lipschitz robustness and standard robustness also global variants make sense: 

\begin{definition} Let $N$  be a NN, $\bar x\in\mathbb R^n$ and $\varepsilon,\delta\in\mathbb Q_{\geq0}$. We say that $N$ has (global) $\varepsilon$-$\delta$-\emph{standard robustness} wrt the metric $d$ iff it has $\varepsilon$-$\delta$-standard robustness in every $\bar x\in\mathbb R^n$:
\[\GSR_d(N,\varepsilon,\delta):=\forall  x,\bar x: d(x-\bar x\vert)\leq\varepsilon\Rightarrow d(N(\bar x)-N(x))\leq\delta.\]

We define $\GLR_d$ analogously, note that it coincides with $L$-Lipschitz continuity independent of $\varepsilon$.
\end{definition}

This equivalence does not hold for local Lipschitz robustness: The function $x\cdot\chi_{\mathbb Q}(x)$ is $\varepsilon$-$1$-Lipschitz robust in 0 for every $\varepsilon$, but not Lipschitz-continuous in any $\varepsilon$-ball around 0.

\begin{remark}
Global arbitrary $\varepsilon$-\emph{classification robustness} $\GACR_d$ for a network $N$, meaning it has arbitrary $\varepsilon$-classification robustness in every input $x\in\R^n$, hardly ever happens at all. We have that only very artificial functions map inputs to different classes and still have global (arbitrary) classification robustness, for between two areas of different classification, there must be a dividing strip of diameter $\varepsilon$ in which both classifications must be choosable. The function $f:\mathbb R\rightarrow \mathbb R^2 $ where
\[f=\begin{pmatrix}
f_1\\
f_2
\end{pmatrix} \text{ where } f_1(x)= \left\{\begin{array}{ll}
        0, & \text{for } x\leq -1\\
        x+1, & \text{for } -1\leq x\leq -\frac12\\
        \frac12, & \text{for } -\frac12\leq x\leq \frac12\\
        x-1, \ \ \ & \text{for } \frac12\leq x\leq 1\\
        1, & \text{for } 1\leq x
        \end{array}\right. , \ \ \ \ \ f_2(x)=1-f_1(x)
  \]
for example is globally 1-classification robust. Such functions however are unlikely for a network of "natural origin" for these areas of equal classification do not occur in usual training methods.

\end{remark}

\begin{theorem}\label{SRGNE}
Let $F$ be a set of activation functions containing $ReLU$. Then:

i) $\GSR_\infty(F)$ reduces to $\NE(F)$.

ii) $\GSR_\infty(F)$, $\LR_\infty(F)$ and $\GLR_\infty(F)$ are co-NP-hard.
\end{theorem}

\begin{proof} $i)$ We reduce an instance $(N,\varepsilon,\delta)$ of $\GSR_\infty(F)$ to an instance $(N',N'')$ of $\NE(F)$ by using the ReLU-function to check whether certain values are positive or not and by forcing the differences $\|\bar x_i-x_i\|_\infty$ to be smaller than $\varepsilon$, again with the use of ReLU-nodes. We construct the network $N'$ as follows: The input $(x,y)=(x_1,...,x_n,y_1,...,y_n)$ has twice the dimension of the input of $N$. On the first half $(x_1,...,x_n)$ we execute $N$ to obtain $N(x)$. For each of the $y_i$ in the second half, we compute
\[f(y_i)=\varepsilon\cdot(ReLU(y_i)-ReLU(y_i-1)-\frac12)\]
by ReLU-nodes in the first hidden layer, a bias of $-\frac12$ in the second and a weight of $\varepsilon$ in the third. Note that $im(f)=[-\varepsilon,\varepsilon]$, we interpret $f(y_i)$ as $\bar x_i-x_i$. We now execute $N$ again, this time on $(x_1+f(y_1),...,x_n+f(y_n))$ to obtain $N(x+\bar x-x)=N(\bar x)$. We compute both $N(x)-N(\bar x)$ and $N(\bar x)-N(x)$ by ReLU-nodes with bias $-\delta$ to obtain 
\[N'(x,y)=(ReLU(N(\bar x)_1-N(x)_1-\delta),ReLU(N(x)_1-N(\bar x)_1-\delta),...,\]
\[ReLU(N(\bar x)_m-N(x)_m-\delta),ReLU(N(x)_m-N(\bar x)_m-\delta)).\]
$\SR^G_\infty(N,\varepsilon,\delta)$ holds iff $N'$ is constantly 0. The second network $N''$ consists of $2n$ input nodes, no hidden layer and $2m$ output nodes with all weights and biases 0. This network will compute 0 independently of the input, so
\[\SR^G_\infty(N,\varepsilon,\delta)\Leftrightarrow \NE(N',N'').\]

$ii)$ We reduce an instance $\varphi$ of 3-SAT with $n$ clauses to an instance $(N,\infty,n-\frac12)$ of the complement of $\GSR_\infty(\{ReLU\})$. The idea is to associate pairs of literals and their negations with pairs of nodes with the property that at least one of these nodes must have the value zero. We then interpret the zero-valued node as false and the other one, that may have any value in $[0,1]$, as true and replace $\land$ and $\lor$ by gadgets that work with the function $\Psi:\mathbb R\rightarrow[0,1]$ defined as $x\mapsto ReLU(x)-ReLU(x-1)$.
Note that it is easily constructable in a ReLU-network and $\Psi\vert_{[0,1]}=id\vert_{[0,1]}$.

For the reduction, introduce for each variable $a$ in$\varphi$ an input variable $\alpha_1$ and use the first two hidden layers to compute $\alpha_2:=\Psi(\alpha_1)$. Use the third hidden layer to propagate $\alpha_2$ as well as to compute $\lnot \alpha_2:=1-\alpha_2$. In the fourth hidden layer compute $\alpha:=2\cdot ReLU(\alpha_2-\frac12)$ for every atom as well as for the negations. Note that $\alpha=0$ or $\lnot \alpha=0$ and the other one can be anything in $[0,1]$. In the fifth and sixth hidden layer compute for every clause $(a,b,c)$ the value $\Psi(\alpha +\beta+ \gamma)$. In the seventh layer, add all nodes from the sixth layer, this is the overall output of the network. If the instance of 3-SAT is solvable, then $N(x)=n$ is reachable via setting the true atoms to 1 and the false ones to 0. The value $N(y)=0$ is always reachable, just assign every variable the value $\frac12$. This means that if the 3-SAT instance is solvable, then $x$ and $y$ are witnesses that $\GSR_\infty(N,\infty,n-\frac12)$ is false. On the other hand, if the 3-SAT instance is not solvable, then for every input $\Psi(\alpha +\beta+\gamma)=0$ for at least one of the clauses, so $N(x)\leq n-1$. Now we always have that $N(y)\geq0$ because $\Psi\geq0$, so $\vert N(x)-N(y)\vert\leq n-1$ and $\GSR_\infty(N,\infty,n-\frac12)$ is true.

We reduce 3-SAT to $\LR_\infty(F)$ and $\GLR_\infty(F)$ in the same way, the instance for $\LR_\infty(F)$ is $(N,\frac12,2n-1,(\frac12,...,\frac12))$ with the same network $N$ as constructed above for $\GSR_\infty(\{ReLU\})$. We chose $n-1<\frac  1L\cdot \varepsilon<n$ so that $\LR_\infty(N,\frac12,2n-1,(\frac12,...,\frac12))$ is equivalent to $\GSR_\infty(N,\infty,n-\frac12)$. This is true because the minimum of $N(x)$ is $0$ attained in $(\frac12,...,\frac12)$ and the maximum is either $n$ if the 3-SAT instance is satisfiable or at most $n-1$ if it is not, taken on the boundary of the unit hypercube. The growth between the minimum and the maximum is linear because the intersection points where ReLU is not differentiable are in $(\frac12,...,\frac12)$ and on the boundary of the unit hypercube, but never in between.

The instance $(N,\frac12,2n-1)$ constructed by the reduction to $\GLR_\infty(F)$ is the same, because no other pair of points $x,\bar x$ will lead to a steeper descent if the network was derived from a SAT-instance that is not satisfiable. This is because the steepest descent must necessarily point to the minimum if the function is linear along the connecting line and no other local minimum exists. $\hfill\blacksquare$
\end{proof}

\begin{corollary}
$\GSR_\infty(\{ReLU\})$, $\NE(\{ReLU\})$, $\CR_\infty(\{ReLU\})$,\\
$\SR_\infty(\{ReLU\})$, $\LR_\infty(\{ReLU\})$ and $\VIP(\{ReLU\})$ are co-NP-complete.
\end{corollary}

\begin{proof}
$\VIP(\{ReLU\})$ is in co-NP by Theorem \ref{VIPcoNP} and $\GSR_\infty(\{ReLU\})$ is co-NP hard by Theorem \ref{SRGNE}, $ii)$. The chain of reductions in between is provided by Theorem \ref{SRGNE} $i)$, Theorem \ref{NECR}, and Theorem \ref{SRCRVIP} $iv),i)$.
$\hfill\blacksquare$
\end{proof}

\section{Network Minimization}\label{Section:mini}
In this section, we want to analyze different criteria on whether a network is unnecessarily huge. The evaluation of a function given as a network becomes easier and faster for small networks, it is therefore desirable to look for a tradeoff between accuracy and runtime performance. In the worst case, a long computation in an enormously big network that was assumed to be more accurate due to its size, could lead to rounding errors in both the training and the evaluation, making the overall accuracy worse than in a small network.

\begin{definition}
Let $F$ be a set of activations, $N$ an $F$-network and $K$ the set of nodes of $N$.

a) $N$ is called minimal, if no other $F$-network with less many nodes than $N$ computes the same function. The decision problem $\mMIN(F)$ asks whether a given $F$-network $N$ is minimal, in that case we write $\mMIN(N)$.

b) Let $Y\subseteq K$ be a subset of the nodes, $G_N$ the underlying graph of $G$ and $G_N'$ its subgraph induced by $K\backslash Y$. The network $N'$ whose underlying is $G_N'$ and where all activations, biases and weights are the same as in the corresponding places in $N$, is said to be obtained from $N$ by deleting the nodes in Y. By abuse of notation we denote it by $N\backslash K:=N'$.

c) A subset $Y\subseteq K$ of hidden nodes of a network $N$ is called unnecessary, if the network $N\backslash Y$ produced by deleting these nodes still computes the same function and otherwise necessary. The problem to decide whether sets of nodes are necessary is \NECE$(F)$, we write \NECE$(N,Y)$ if $Y$ is necessary.

d) If \NECE$(N,Y)$ for all non-empty subsets $Y\subseteq K$ of nodes, we write \ANECE$(N)$, this also defines a decision problem in the obvious way.
\end{definition}

It obviously holds that $\mMIN(N)\Rightarrow\mANECE(N)$, the reverse direction however does not hold in general:

\begin{example}
A network can contain only necessary nodes and still not be minimal. Consider for example the $\{id,ReLU\}$-network $N$ that consists of an input node $x$, two hidden nodes $y_{1,1}=ReLU(x)$ and $y_{1,2}=ReLU(-x)$ and one output node with weights 1 and -1 and bias 0, so it computes $y_2=id(y_{1,1}-y_{1,2})=x$. The network represents the identity map and is not minimal, because the smallest $\{id,ReLU\}$-network $M$ computing the identity consists of an input node, no hidden nodes and one output node with activation $id$, weight 1 and bias 0. However every non-empty subset of hidden nodes in $N$ is necessary; if we delete $y_{1,1}$, the remaining network computes the negative part $x^-=min\{0,x\}$ of the input, if we delete $y_{1,2}$ it computes the positive part and if we delete both it computes the zero-function for the last node has no more incoming nodes at all and thus returns the empty sum. We call this network $Q$. \footnote{This also illustrates that the smallest network $M$ that is equivalent to $N$ is in general not isomorphic to a substructure of $N$.}

Also consider the net $K$ computing the zero-function in the following way: In the first hidden layer, compute $y_{1,1}=y_{1,2}=ReLU(x)$, the output-node is connected to both with weights 1 and -1, respectively, and bias 0, so it computes $0=ReLU(x)-ReLU(x)$. In $K$, both hidden nodes on their own are necessary, the subset containing both however is not, because when deleting those two, we remain with the network $Q$ that also computes the zero-function.

\end{example}

\begin{theorem}
Let $id\in F$, then the complement of \NECE$(F)$ and \NE$(F)$ are linear-time equivalent.
\end{theorem}

\begin{proof}
The first reduction is trivial, 
\[\mNECE(N,\{y_1,...,y_k\})\Leftrightarrow\lnot\mNE(N,N\backslash\{y_1,...,y_k\})\]

For the other direction, let $M,N$ be two $F$-networks. We construct an instance of \NECE$(F)$ that is equivalent to $\lnot\mNE(N,M)$ as follows:

First, we merge $N$ and $M$ into a network $P$ that computes $N(x)$ and $M(x)$, then computes $y=N(x)-M(x)$ in the last hidden layer and then propagates the value $P(x)=id(y)$ to the output layer. Now if $y$ is deleted, $P\backslash\{y\}\equiv0$, because then the output node would have no incoming edges at all, so it would compute the identity over the empty sum, which is 0. It follows that $y$ is unnecessary iff $P\equiv0$, which in turn is the case iff $N\equiv M$ meaning \NE$(N,M)$, so we have that $\lnot\mNE(N,M)\Leftrightarrow\mNECE(P,\{y\})$. $\hfill\blacksquare$
\end{proof}

\begin{theorem}
Let $F$ only contain semi-linear functions. Then
\begin{description}
\item[i)] \MIN$(F)$ is in $\Pi_2^P$.
\item[ii)] \ANECE$(F)$ is in $\Pi_2^P$.
\end{description}
\end{theorem}

\begin{proof}\begin{description}
\item[i)] We argue that co-\MIN\ is in $\Sigma_2^P$. The algorithm non-deterministically guesses a smaller representation $M$ of the function/network $N$ and then proves that \NE$(N,M)$, which is co-NP-complete, see \cite{adrian}.
\item[ii)] In the same way, observe that co-\ANECE$(F)$ can be solved by non-\\
deterministically guessing an unnecessary subset $Y$ and then showing that the network  remaining after deleting $Y$ is equivalent to the original net.$\hfill\blacksquare$
\end{description}
\end{proof}

\begin{theorem}
If id is among the activations, \NE\ reduces to \ANECE\ in polynomial time.
\end{theorem}

\begin{proof}
Let $(N,M)$ be an instance of \NE. First, merge both $N$ and $M$ into one network by subtracting their corresponding output nodes. Then delete all nodes that are constantly zero. These are detected layer-wise starting with the nodes right after the input layer. For each node $n$, delete all nodes that do not affect it in the sense that they are in subsequent layers, in the same layer, or in a previous layer but with no connection or connection with weights 0 towards $n$. Then allocate one additional id-node in the following layer and connect it to $n$ with weight 1 and no bias.$\hfill\blacksquare$
\end{proof}

\section{Conclusion}
We examined the computational complexity of various verification problems such as robustness and minimality for neural networks in dependence of the activation function used. We explained the connection between the distance function in use and the complexity of the resulting problem. We provided a framework for minimizing networks and examined the complexities of those questions as well.

Further open questions are:

\begin{description}
\item[1.)] Is $\LR_1(ReLU)$ complete for co-NP? Are $\LR_1(F)$ and $\LR_\infty(F)$ in general reducible on each other?

\item[2.)] Do sets described by functions computed by networks using only $ReLU$ (or only semi-linear activations) have more structure than just any semi-linear set? Can this structure be exploited to calculate or approximate the measure of such a set quicker than with the known algorithms?

\item[3.)] Are \ANECE\ and \MIN$(F)$ $\Pi_2^P$-complete if $F$ only contains semi-linear functions? 

\item[4.)] Several results rely on the identity activation. Can this be omitted? Or is there a problem that becomes substantially easier when the identity map is dropped? 

\end{description}
Acknowledgment: I want to thank Klaus Meer for helpful discussion and the anonymous referees for various hints improving the writing.

\bibliography{literatur}

\bibliographystyle{plain}

\end{document}